\pdfoutput=1
\documentclass[10]{article}
\usepackage{arxiv}
\usepackage{natbib}
\usepackage{delimset, diffcoeff, interval}
\usepackage[colorlinks=true, urlcolor=blue, linkcolor=blue, citecolor=blue]{hyperref}
\usepackage{amsmath, amsthm, amsfonts, mathtools, nicematrix, bm, cancel, siunitx, tensor}
\usepackage[ruled,vlined, algo2e]{algorithm2e}
\SetKwComment{Comment}{$\triangleright$\ }{}
\usepackage{graphicx, booktabs, accents, enumerate}
\usepackage{import, float, pdfpages, transparent, xcolor}
\usepackage[hang]{footmisc} 
\usepackage[capitalise]{cleveref}
\usepackage{authblk}
\usepackage{printlen} 
\usepackage{mwe} 
\usepackage{graphbox}
\usepackage{subcaption}

\theoremstyle{definition}

\theoremstyle{plain}
\newtheorem{theorem}{Theorem}[section]

\newtheorem{lemma}[theorem]{Lemma}

\theoremstyle{remark}

\newtheorem{remark}[theorem]{Remark}

\intervalconfig{soft  open  fences ,separator  symbol=;,}
\newcommand\numberthis{\addtocounter{equation}{1}\tag{\theequation}} 
\newcommand{\dkl}{\mathcal{D}_{\mathrm{KL}}} 
\newcommand{\defn}{\coloneqq} 

\newcommand{\vgf}{v^{\text{GF}}} 
\newcommand{\reals}{\mathbb{R}} 
\newcommand{\natrl}{\mathbb{N}} 
\newcommand{\dfl}[1]{\: d #1} 

\newcommand{\lap}{\Delta} 

\newcommand{\Iddd}{I_{d \times d}}
\newcommand{\Indnd}{I_{Nd \times Nd}}
\newcommand{\gauss}{\mathcal{N}} 

\newcommand{\rkhsd}{\mathcal{H}^d}
\newcommand{\rkhs}{\mathcal{H}}
\newcommand{\vsvgd}{v^{\text{SVGD}}}

\newcommand{\vlrho}{\mathcal{L}_{\rho}^2} 

\newcommand{\vsvn}{v^{\text{SVN}}}

\newcommand{\dsvn}{D^{\text{SVN}}}

\newcommand{\order}{\mathcal{O}}
\newcommand{\vdet}{v^{\textrm{DET}}}
\newcommand{\vstc}{v^{\text{STC}}}
\newcommand{\realsnd}{\reals^{Nd}}
\newcommand{\realsndnd}{\reals^{Nd \times Nd}}
\newcommand{\realsd}{\reals^{d}}
\newcommand{\kgram}{\bar{k}}
\newcommand{\wnfhess}{\tilde{H}}

\newcommand{\showfontsize}{\f@size{} pt}

\newcommand{\ibar}{%
  \text{\ooalign{\hidewidth -\kern-.1em-\hidewidth\cr$i$\cr}}%
}

\DeclareMathOperator*{\expv}{\mathbb{E}} 

\DeclareMathOperator{\unif}{Unif}

\newcolumntype{+}{>{\global\let\currentrowstyle\relax}}
\newcolumntype{^}{>{\currentrowstyle}}

\usepackage[normalem]{ulem}

\allowdisplaybreaks
\begin{document}

\title{A stochastic Stein variational Newton method}

\author[1]{Alex Leviyev}
\author[2]{Joshua Chen}
\author[3]{Yifei Wang}
\author[2]{Omar Ghattas}
\author[1]{Aaron Zimmerman}
\affil[1]{Center for Gravitational Physics, University of Texas at Austin}
\affil[2]{Oden Institute, University of Texas at Austin}
\affil[3]{Department of Electrical Engineering, Stanford University}

\date{\today}
\maketitle

\begin{abstract}
Stein variational gradient descent (SVGD) is a general-purpose optimization-based sampling algorithm that has recently exploded in popularity, but is limited by two issues: it is known to produce biased samples, and it can be slow to converge on complicated distributions.
A recently proposed stochastic variant of SVGD (sSVGD) addresses the first issue, producing unbiased samples by incorporating a special noise into the SVGD dynamics such that asymptotic convergence is guaranteed.
Meanwhile, Stein variational Newton (SVN), a Newton-like extension of SVGD, dramatically accelerates the convergence of SVGD by incorporating Hessian information into the dynamics, but also produces biased samples.
In this paper we derive, and provide a practical implementation of, a stochastic variant of SVN (sSVN) which is both asymptotically correct and converges rapidly.
We demonstrate the effectiveness of our algorithm on a difficult class of test problems---the Hybrid Rosenbrock density---and show that sSVN converges using three orders of magnitude fewer gradient evaluations of the log likelihood than its stochastic SVGD counterpart.
Our results show that sSVN is a promising approach to accelerating high-precision Bayesian inference tasks with modest-dimension, $d\sim\order(10)$.\footnote{Our code is available at \url{https://github.com/leviyevalex/sSVN}}
\end{abstract}

\section{Introduction}
The goal of Bayesian inference in data analysis is to infer probability distributions (posteriors) for model parameters, given a dataset \cite{van2021bayesian}.
It is a powerful framework for parameter estimation, but poses significant computational challenges. The posterior must generally be evaluated numerically and the parameter space may be high dimensional.
Consequently, quantities of interest---such as moments and various other integrals of the posterior---are non-trivial to calculate.

Markov chain Monte Carlo (MCMC) is a widely used technique to approximate such integrals \cite{speagle2020conceptual}.
By defining a suitable Markov chain, one may in principle draw i.i.d.\thinspace samples from the posterior and thus calculate arbitrary integrals of the posterior with $\order(1/\sqrt{N})$ accuracy, where $N$ is the sample size drawn.
In many applications however, MCMC algorithms may be unacceptably slow to converge.

Variational inference (VI), on the other hand, is a technique which trades accuracy for speed. VI algorithms approximate posterior integrals by solving an optimization problem \cite{2017}.
A popular example of such an approach is Stein variational gradient descent (SVGD), which is a non-parametric VI algorithm which implements a form of functional gradient descent \cite{liuSteinVariationalGradient2016a}.
Likewise, Stein variational Newton (SVN) \cite{detommasoSteinVariationalNewton2018a,chen2020projectedSVN} extends SVGD by implementing a form of functional Newton descent, dramatically accelerating convergence to the posterior at the price of additional work per iteration.

Although VI algorithms were originally developed to accelerate Bayesian inference in machine learning, many areas of science and engineering similarly rely on Bayesian inference, albiet with a modest number of dimensions $d \sim \order(10)$.
This is the case, for example, in many astrophysical applications such as cosmology~\citep[e.g.][]{Planck:2018vyg} and gravitational wave astronomy \citep[e.g.][]{LIGOScientific:2021djp}.
VI offers the possibility of significantly accelerating inference, and is thus a tempting candidate to investigate further \cite{Gunapati2018VariationalIA}.
However, the inexact nature of VI limits its utility: especially in fields where high-precision posterior sampling is required.

Ideally one would construct an algorithm with the speed and flexibility of VI, but with the asymptotic convergence gaurantees of MCMC.
Indeed, an example of such an algorithm was proposed in \citet{gallegoStochasticGradientMCMC2020}, which showed that SVGD may be asymptotically corrected by adding in a special noise term, which leads to a stochastic SVGD (sSVGD) algorithm.
Compared to naive MCMC algorithms, sSVGD proposals are constructed by taking an SVGD descent step and centering a Gaussian at that updated point.
In this sense sSVGD proposals are well informed, ``aggressive,'' in the sense that the proposal is centered away from the current point, and consequently expected to perform more robustly compared to other MCMC proposals.
In this paper we show that it is possible to do the same with SVN, thus yielding a significantly faster, yet still asymptotically correct MCMC algorithm.

Our contributions are as follows:

\begin{itemize}
\item We show that adding both a \textit{stochastic} and a \textit{deterministic} correction to the SVN dynamics forms a Markov chain with asymptotically correct stationary density.
\item We introduce a practical implementation of sSVN, including a Levenberg-like damping term which improves the stability and globalization properties of the flow.
\item We demonstrate that our algorithm (sSVN) has excellent posterior reconstruction properties, and equilibrates with $\order(1000)$ times fewer gradient evaluations of the log-likelihood than sSVGD.
\end{itemize}

The outline of the paper is as follows.
We discuss related Newton-based MCMC methods in \cref{sec:related-work}.
In \cref{sec:background} we review SVGD,
and the diffusive MCMC recipe which motivates our proposed modifications to the SVN dynamics.
We derive sSVGD from the recipe, and then briefly review the standard SVN algorithm.
In \cref{sec:ssvn} we introduce our sSVN algorithm.
Finally we present our numerical results in \cref{sec:numerics}, and offer our conclusions and outlook in \cref{sec:conclusion}.

\section{Related work}\label{sec:related-work}
In this paper we propose and provide a practical implementation of a stochastic variant of Stein variational Newton (sSVN).
The sSVN update utilizes a Newton direction, and adds a special noise term found from the discretization of a certain stochastic differential equation, discussed in \cref{sec:ssvn}.
This resembles stochastic Newton (SN) \cite{martin2012stochastic}, which utilizes as a proposal a Newton step with noise added.
Whereas SN performs a Newton method on the posterior directly, sSVN simulates a Wasserstein-Newton flow (WNF) of the Kullback-Leibler divergence
\cite{wangInformationNewtonFlow2020, liuUnderstandingMCMCDynamics2019} (see \cref{sec:WNF}).
Further, this simulation of WNF, contrary to SN, yields dynamics for an interacting many-particle system, as opposed to SN, which yields dynamics for a single particle.
For more on Newton methods in MCMC, see \citet{martin2012stochastic, qi2002hessian, zhang2011quasi, simsekli2016stochastic} and references therein.

\section{Background}\label{sec:background}
\subsection{Standard SVGD}
SVGD--originally introduced in \citet{liuSteinVariationalGradient2016a}--is an attractively simple algorithm designed to sample from posterior distributions $\pi$ over $\realsd$.\footnote{
We choose $\realsd$ for simplicity.
Extensions to finite \cite{liu2017riemannian} and other infinite-dimensional manifolds \cite{jia2021stein} have been developed as well.}
Beginning with any $N \in \natrl^+$ number of ``particles'' at positions $z_m\in\realsd$, with $1 \le m \le N$, SVGD evolves these particles through the dynamics
\begin{equation}
\label{eq:svgd-dynamics}
\diff{z_m}{t} = \vsvgd(z_m) \,,
\end{equation}
where the particles are coupled through a velocity field $\vsvgd: \realsd \to \realsd$ defined by
\begin{equation}
\label{eq:SVGD-direction}
\vsvgd(z_m) = \frac{1}{N} \sum_{n=1}^N \brk[s]!{k(z_m, z_n) \nabla \ln \pi(z_n) + \nabla_2 k(z_m, z_n)} \,,
\end{equation}
where $k:\realsd \times \realsd \to \reals$ is a kernel describing the interaction between particles, and $\nabla_2 k(z_m, z_n)$ represents the gradient of the kernel with respect to the second argument $z_n$.
An Euler discretization of \cref{eq:SVGD-direction} in fact yields the SVGD algorithm which is summarized in \cref{algo:svgd}.
\begin{algorithm2e}
\SetAlgoLined
\KwIn{Initialize particles $\brk[c]{z^1_m}_{m=1}^N$, timestep $\tau>0$}
 \For{$l = 1, 2, \ldots, L$}{
Calculate $\vsvgd$ \;
$z_m^{l+1} \leftarrow z_m^l + \tau \vsvgd(z_m^l) \quad \forall m$ \;
 }
 \caption{SVGD}\label{algo:svgd}
\end{algorithm2e}

There are several useful features of the velocity field \cref{eq:SVGD-direction}.
First, these dynamics minimize the KL-Divergence $\dkl(\rho, \pi)$ where $\rho(x) = \sum_m \delta(x - z_m)$ represents the empirical measure over $\realsd$ of an ensemble $\brk[c]{z_m}_{m=1}^N$, and $\pi$ is the usual posterior.
Specifically, these dynamics can be shown to be associated with a gradient descent of $\dkl$.
This gives us a loose guarantee that once the particles equilibrate they will approximate the first few moments of the posterior reasonably well.

\subsection{sSVGD}\label{sec:ssvgd}
Although SVGD has has shown promise in application \cite{pmlr-v97-gong19b,10.1093/gji/ggaa170,pinder2021stein,chen2020projectedSVGD}, it is not without its downsides.
For example, SVGD provides an inherently biased estimate of a distribution and is known to underestimate the variances \cite{anonymous2022understanding}.
Mischaracterizations of this type may make SVGD unsuitable for applications where high precision posterior reconstruction is necessary.
Although several mean-field and asymptotic convergence results have been established \cite{korba2021nonasymptotic,liu2018stein,liu2017stein,duncanGeometrySteinVariational2019}, it would nonetheless be desirable to have finite particle, asymptotic convergence guarantees.
sSVGD, first introduced in \citet{gallegoStochasticGradientMCMC2020}, addresses these issues by adding a (computationally negligible) Gaussian noise into the SVGD dynamics: transforming the dynamics into a Markov chain with asymptotic guarantees in the continuous time limit.
Of considerable interest is that adding noise allows us to begin collecting samples after a burn in period, as opposed to evolving a large number of particles from the onset.
Lastly, in discrete time this scheme supports a Metropolis-Hastings correction which in theory eliminates \textit{all} sources of bias present in the dynamics, leading to a truly ``correct" sampling scheme.

In this section we review how to derive this SVGD noise term and review the sSVGD algorithm.
Before doing so we briefly discuss the MCMC recipe framework and present results needed to motivate both sSVGD and later sSVN.
\paragraph{Diffusive MCMC over configuration space}
Suppose $\brk[c]{z_m}_{m=1}^N$ denotes an ensemble of $N$ particles, and let us define the associated \textit{configuration space} of the ensemble by $X \defn \brk[c]{z: z=[z_1^{\top}, \ldots, z_N^{\top}]^{\top}}$. Furthermore, let us lift the score function into this configuration space by defining $\nabla \ln \pi(z) \in \realsnd$ such that $\nabla \ln \pi(z)=[\nabla^{\top} \ln \pi (z_1), \ldots, \nabla^{\top} \ln \pi(z_N)]^{\top}$. Then using the results of \citet{maCompleteRecipeStochastic2015} one many construct a Markov chain over $X$ with invariant density $\Pi_{n=1}^N \pi(z_n)$ by discretizing the following Ito equation
\begin{equation}\label{eq:ito-recipe}
\dfl{z} = f(z) \dfl{t} + \sqrt{2 D(z)} \dfl{B},
\end{equation}
where the \textit{drift} $f:\realsnd \to \realsnd$ is given by
\begin{align*}
f(z) = \brk[s]!{D(z) + Q(z)} \nabla \ln \pi(z) + \nabla \cdot \brk[s]!{D(z) + Q(z)}\,,
\end{align*}
where $D, Q: \realsnd \to \realsndnd$ are a positive semi-definite \textit{diffusion matrix} and skew symmetric \textit{curl matrix} respectively, the divergence is understood to sum into the second index of the matrices, and $B$ is a $Nd$-dimensional Brownian motion.
In practice an Euler-Maruyama discretization of \cref{eq:ito-recipe} is taken, which yields the following Markov chain\footnote{
Note we have corrected a sign in the corresponding equation in \citet{maCompleteRecipeStochastic2015}.}:
\begin{align*}
\label{eq:eps-discretization-sampler}
z^{l+1} \leftarrow & z^l + \tau \brk[s]!{D(z^l) + Q(z^l)}\nabla \ln \pi(z^l) + \nabla \cdot \brk[s]!{D(z^l) + Q(z^l)} +  \gauss(0, 2 \tau D(z^l)) \,,
\numberthis
\end{align*}
where $\gauss(\mu, \Sigma)$ denotes a Gaussian random variable with mean $\mu\in\realsnd$ and covariance $\Sigma \in \realsndnd$.

With the stage set, let us define the matrix function $K: \realsnd \to \realsndnd$ with
\begin{equation}\label{eq:svgd-diffusion-matrix}
K(z) \defn \frac{1}{N}
\begin{pNiceMatrix}[]
k(z_1, z_1) \Iddd & \cdots & k(z_1,z_N) \Iddd\\
\vdots & \ddots & \vdots\\
k(z_N,z_1) \Iddd & \cdots & k(z_N,z_N) \Iddd
\end{pNiceMatrix} \,,
\end{equation}
then it follows directly that \cite{gallegoStochasticGradientMCMC2020}
\begin{lemma}[SVGD recast into MCMC recipe form]\label{prop:SVGD-MCMC-equivalence}
Suppose that $\brk[c]{z_m}_{m=1}^N$ is an ensemble of particles, and $k$ is a kernel such that for every particle $z_m$ in the ensemble $\nabla_1 k(z_m, z_m) = 0$ holds. Then SVGD may be expressed as
\begin{equation}
\label{eq:SVGDrecipeform}
\diff{z}{t} = K(z) \nabla \ln \pi(z) + \nabla \cdot K(z) \,
\end{equation}
over configuration space X.
\end{lemma}
\begin{proof}
Follows directly from \cref{eq:SVGD-direction}. See the Appendix for details.
\end{proof}
Moving forward, we suppress the dependency of $K$ on $z$ when convenient.
\begin{remark}
\cref{eq:SVGDrecipeform} is of the form \cref{eq:ito-recipe} with diffusion matrix $D=K$ and curl matrix $Q=0$. Note, however, that that \cref{eq:SVGDrecipeform} only takes into account the drift.
\end{remark}

\paragraph{Cost of calculating noise}
\cref{eq:SVGDrecipeform} immediately suggests adding a noise term drawn from $\gauss(0,2K)$.
Naively, such a draw would require calculating the lower triangular Cholesky decomposition of an $Nd \times Nd$ matrix, which may be expensive.
Instead we may exploit the fact that $P K P^{\top} = D_K$ where $D_K$ is the block-diagonal matrix
\begin{equation}\label{eq:DK}
D_K \defn
\frac{1}{N}
\left[
\begin{array}{ccc}
\kgram &  & \\
 & \ddots & \\
 &  & \kgram \\
\end{array}
\right] ,
%
%
\end{equation}
where $\kgram \in \reals^{N \times N}$ is the kernel gram matrix with components $\kgram_{mn} = k(z_m, z_n)$, and $P$ is the permutation matrix whose action $Pz$ on a vector $v \in \realsnd$ performs a transformation of basis from one where the coordinates of each particle are listed sequentially to one where the first coordinate of each particle is listed, then the second, and so on.\footnote{See \cref{fig:coordinate-to-particle}, and for a proof of this fact see \cref{sec:Proofs}.}
This yields
\begin{align*}
\vstc &\sim \gauss(0, 2K) \\
&\sim \sqrt{2} P^{\top} P \gauss(0, K) \\
&\sim \sqrt{2}P^{\top}\gauss(0, D_K) \\
&\sim\sqrt{2} P^{\top}L_{D_K} \gauss(0, \Indnd) \,, \numberthis \label{eq:v-stc-svgd}
\end{align*}
where $L_{D_K}$ denotes the lower triangular Cholesky decomposition of $D_K$, and only requires calculating the lower triangular Cholesky decomposition of the kernel gram matrix, $L_{\kgram}$.
Since $\kgram$ is independent of $d$ and in practice $N$ is modest, evaluating this noise is computationally negligible.
Finally, we note that the action of $P^{\top}$ is to perform a simple tensor reshaping, and is thus trivial to implement.
sSVGD is thus a simple modification of SVGD, and is summarized in \cref{algo:ssvgd}.
For an equivalent description to \cref{algo:svgd} set $\vstc=0$.
\begin{algorithm2e}
\SetAlgoLined
\KwIn{Initialize ensemble $z^1$, timestep $\tau>0$}
 \For{$l = 1, 2, \ldots, L$}{
Calculate $\vsvgd(z^l)$ using \cref{eq:SVGD-direction}\;
Calculate $\vstc(z^l)$ using \cref{eq:v-stc-svgd}\;
$z^{l+1} \leftarrow z^l + \tau \vsvgd(z^l) + \sqrt{\tau} \vstc(z^l)$ \;
 }
 \caption{Stochastic SVGD}\label{algo:ssvgd}
\end{algorithm2e}

\subsection{Standard SVN}

The SVN algorithm \cite{detommasoSteinVariationalNewton2018a} extends SVGD by solving the following linear system for coefficients $\alpha \in \realsnd$
\begin{equation}\label{eq:svn-linear-system}
H(z) \alpha =  v^{\text{SVGD}}(z) \,,
\end{equation}
where $H: \realsnd \to \realsnd \times \realsnd$ denotes the \textit{SVN-Hessian} and takes the form

\begin{equation}\label{eq:SVN-Hessian}
H(z) \defn
\left[
\begin{array}{ccc}
h^{11} & \cdots & h^{1N}\\
\vdots & \ddots & \vdots\\
h^{N1} & \cdots & h^{NN}\\
\end{array}
\right] \,,
\end{equation}


and for $1 \le m, n \le N$ the blocks $h^{mn} \in \reals^{d \times d}$ are defined by
\begin{align*}
h^{mn} \defn \frac{1}{N}\sum_{p=1}^N \big[&- k(z_p, z_m) k(z_p, z_n) \nabla^2\ln \pi(z_p)
+ [\nabla_1 k(z_p, z_n)] \nabla^{\top}_1 k(z_p, z_m) \big] \,. \numberthis \label{eq:blocks}
\end{align*}
Finally, one forms the SVN direction by
\begin{equation}\label{eq:svn-direction}
\vsvn(z) = N K(z) \alpha(z) \,,
\end{equation}
and uses this velocity field to drive the system of particles in place of $\vsvgd$.
From here we suppress the dependency of $H$ on $z$ when convenient.

\section{Stochastic SVN}\label{sec:ssvn}

\subsection{Correcting SVN}

We can derive a stochastic SVN algorithm by observing that the SVN velocity field can be rewritten as
\begin{align*}
\vsvn &= N K H^{-1} \vsvgd \\
&= N K H^{-1} K \nabla \ln \pi + N K H^{-1}\nabla \cdot K \,, \numberthis
\label{eq:pre-corrected-svn}
\end{align*}
using \cref{eq:svgd-diffusion-matrix,eq:svn-linear-system,eq:svn-direction}.
From this we can see that by defining a diffusion matrix $\dsvn =N K H^{-1} K$, the SVN dynamics can be brought into the form of an MCMC recipe with $Q = 0$ by adding both a deterministic and stochastic correction.
Indeed, $\dsvn$ is symmetric by construction, and is positive definite if $H$ is positive definite. Note that this is not true in general, even for log convex problems, and thus requires modifying $H$.
The first term in \cref{eq:blocks} is positive definite if $\nabla^2 \ln \pi(z_m)$ is positive definite for every $1 \le m \le N$.
This may be accomplished, for example, by replacing the Hessian of the log-likelihood with a Gauss-Newton approximation.
Likewise, one may guarantee that the second term in \cref{eq:blocks} is positive semi-definite by taking the block diagonal approximation.
This ensures that each block takes the form $(1/N)\sum_n \nabla_1 k(z_m, z_n) \nabla_1^{\top} k(z_m, z_n)$, and ensures each block is positive semi-definite.

With this insight and the results of \citet{maCompleteRecipeStochastic2015} we have
\begin{theorem}[Asymptotically correct SVN]\label{thm:SVN-recipe}
The following Ito equation
\begin{align}
\dfl{z} = (\vsvn + \vdet)\dfl{t} + \sqrt{2 \dsvn} \dfl{B} \label{eq:svnsde}
\end{align}
has invariant distribution $\Pi_{n=1}^N \pi(z_n)$ over $X$, where
$\vdet \in \realsnd$ is defined by
\begin{equation}\label{eq:deterministic-correction}
\vdet_a = N K_{bc} \nabla_c (K_{ae} H_{eb}^{-1}) \,,
\end{equation}
and $B \in \realsnd$ is a standard Brownian motion.
\end{theorem}

In the above we have used index notation to express $\vdet$, with repeated indices summed. As was done in \citet{maCompleteRecipeStochastic2015} and \citet{gallegoStochasticGradientMCMC2020}, a simple Euler-Maruyama discretization of \cref{eq:svnsde} suffices as an MCMC proposal.
Namely,
\begin{equation}\label{eq:sSVNDynamics}
z^{l+1} = z^{l} + (\vsvn + \vdet) \tau + \sqrt{\tau} \vstc
\end{equation}
where $\vstc \sim  \gauss(0,2 \dsvn)$.
This forms the basis of a practical sSVN algorithm.

\subsection{Practical algorithm}\label{subsec:practical-algorithm}
While \cref{eq:sSVNDynamics} provides an MCMC recipe based on SVN, several modifications are needed for a stable, fast, and practical algorithm.

\paragraph{Levenberg damping}
For ensemble configurations that are far from equilibrium it is not uncommon for Newton's method to lead to bad steps.
Furthermore, solving the system \cref{eq:svn-linear-system} may be difficult because of a poorly conditioned $H$.
To address these, we introduce a \textit{Levenberg-like damping} to the Hessian. Namely, we use
\begin{equation}\label{eq:damped-hessian}
H_{\lambda} = H + \lambda N K
\end{equation}
rather than $H$ in our diffusion matrix.
In addition to improving the conditioning of $H$, for sufficiently large $\lambda>0$ this damping procedure acts as a simple heuristic to preserve the fast equilibration of Newton while inheriting the stability of gradient descent.
Indeed, consider replacing $\dsvn$ with
\begin{align*}
\dsvn \leftarrow N K \brk[r]{H + \lambda NK}^{-1} K \,.
\end{align*}
Then $\lim_{\lambda \uparrow \infty} \dsvn = K/\lambda$.
Thus, as $\lambda$ increases, the diffusion matrix approaches that of SVGD, and thus hybridizes SVN with SVGD.
We observe that such a modification improves the numerical stability and sample quality of the Newton flow.

\paragraph{Asymptotic correctness}
Note that $\vdet$ accounts for local curvature variations, and appears in other higher order dynamics such as Riemannian Hamiltonian Monte Carlo \cite{girolami2011riemann}, Riemannian Langevin dynamics \cite{NIPS2013_309928d4}, Riemannian Stein variational Gradient descent \cite{liu2017riemannian}, and stochastic Newton \cite{martin2012stochastic}.
However, calculating $\vdet$ requires third order derivatives of the posterior and is a significant computational burden.
As done in previous works, we propose neglecting $\vdet$ in the dynamics.
In practice, we were still able to collect high quality samples in the numerical experiments presented in \cref{sec:numerics} while simply neglecting $\vdet$.
\begin{remark}[]\label{}

If it becomes necessary to correct any bias introduced by neglecting $\vdet$, we propose two possible solutions.
The first is to introduce a ``damping schedule" for $\lambda$ similar to \citet{dangelo2021annealed}.
If $\lambda$ is made sufficiently large towards the end of the flow, then the asymptotic guarantees of sSVGD will be inherited.
The second solution is to incorporate a Metropolis-Hastings correction.
 \end{remark}
\paragraph{Noise addition}
The discretization of the noise term in \cref{eq:svnsde} may be calculated with
\begin{align*}
\vstc &= \sqrt{2N}\gauss(0, K (H + \lambda N K)^{-1} K) \\
&= \sqrt{2N}K L^{-\top}_{H_{\lambda}} \gauss(0,\Indnd)\numberthis \,,
\label{eq:svn-noise-injection}
\end{align*}
where $L_{H_\lambda}$ represents the lower triangular Cholesky decomposition of the damped Hessian and $(\cdot)^{-\top}$ is the inverse transpose operation.
The proposed algorithm is summarized in \cref{algo:ssvn}.
\begin{algorithm2e}
\SetAlgoLined
\KwIn{Initialize ensemble $z^1$, $\lambda>0$, $\tau>0$}
 \For{$l = 1, 2, \ldots, L$}{
Calculate damped Hessian $H_{\lambda}(z^l)$ \cref{eq:damped-hessian} \;
Calculate Cholesky decomposition $L_{H_{\lambda}}$ \;
Cholesky solve for $\alpha$ and form $\vsvn$ \cref{eq:svn-linear-system}\; 
Triangular solve and form $\vstc(z^l)$ \cref{eq:svn-noise-injection}\;
$z^{l+1} \leftarrow z^l + \tau \vsvn(z^l) + \sqrt{\tau} \vstc(z^l)$ \;
 }
 \caption{Stochastic SVN (Cholesky)}\label{algo:ssvn}
\end{algorithm2e}

\subsection{Complexity analysis}
The proposed algorithm stores the Hessian at a space complexity of $\order(N^2d^2)$, and performs a Cholesky decomposition with time complexity $\order(N^3 d^3)$.
The Cholesky solve for $\vsvn$ and the triangular solve for $\vstc$ does not affect the overall scaling, as the solves take $\order(N^2 d^2)$ time.
Thus the method is bound either by the cost of decomposing $H$, or by gradient and Hessian evaluations of the log-likelihood.
In \cref{sec:scaling} we propose an alternative numerical scheme utilizing a Krylov solver--- requiring only the formation of matrix-vector products, and thus does not require forming $H$.
Further, the Krylov iterations may be terminated when appropriate norm conditions are reached to improve time-scaling.

\section{Numerical experiments}\label{sec:numerics}

Here we present numerical comparisons of our sSVN algorithm to other methods, primarily using the Hybrid Rosenbrock distribution \cite{pagani2020ndimensional} as our posterior, in two, five, and ten dimensions.
This distribution can be flexibly adjusted to be highly correlated in each dimension, with long tails that are challenging to sample; in addition, it is designed such that we can draw a large number of unbiased i.i.d.\thinspace samples ({\it Truth} or {\it ground truth} in the following) to compare to our results. Using this ground truth we evaluate the performance of sSVN using several metrics such as: maximum mean discrepancy (MMD) \cite{JMLR:v13:gretton12a}, comparing means $\mu_i$ and diagonal of the covariance $\sigma_{ii}$ in each dimension $1 \le i \le d$, P-P plots \cite{doi:10.1080/00031305.1991.10475759,loy2015variations}, and corner plots.
All test cases that follow use the Gauss-Newton approximation to ensure that the SVN Hessian $H$ is positive definite.
Likewise, we use the kernel \cite{detommasoSteinVariationalNewton2018a}
\begin{equation}\label{eq:detomasso-kernel}
k(x, y) = e^{-\frac{1}{2h} (x-y)^{\top} M (x-y)} \,,
\end{equation}
with a fixed {\it bandwidth} $h=d$, and {\it metric} $M$ taken to be either the average Gauss-Newton Hessian over the set of particles unless specified otherwise.
Finally, in all sSVN experiments we use a constant damping of $\lambda = 0.01$ for simplicity.
\subsection{Two-dimensional test cases}

\begin{figure}[tb]
\centering
\includegraphics{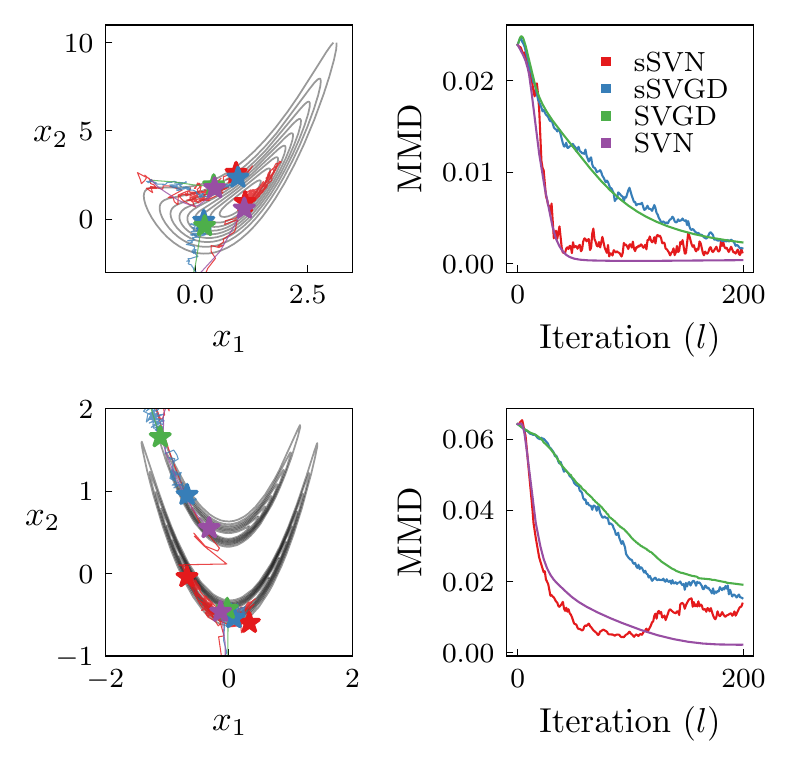}
\caption{Two-dimensional experiments. Top row is a Hybrid-Rosenbrock density with parameters $n_2=1$, $n_1=2$, $a=0.5$, $b=0.5$, and the bottom row is a double-banana.
The left column illustrates the geometry of the posterior, and displays the trajectory one of the particles follows along the various dynamics. The right column displays the evolution of the MMD using $N=100$ particles and 300 samples from the ground truth.}
\label{fig:2d-exp}
\end{figure}

We begin with two low dimensional toy problems.
The first is the Hybrid Rosenbrock distribution \cite{pagani2020ndimensional} with parameters $n_2=1$, $n_1=2$, $a=0.5$, and $b=0.5$, and the second is a double banana as described in \citet{detommasoSteinVariationalNewton2018a}.
We use $N=100$ particles, initially sampled from $\unif[-6, 6]$, and evolve for $L=200$ iterations with step size $\tau=0.1$.

\cref{fig:2d-exp} illustrates particle trajectories traced out by SVGD, SVN, and their stochastic counterparts.
Of particular interest is that sSVN appears more efficient in its exploration of the posterior---in the sense that it explores more of the posterior in the same amount of time.
In addition, the sSVN noise appears to facilitate mode hopping.
We observe that the deterministic and stochastic counterparts make similar progress in MMD.

The settings for the Hybrid Rosenbrock density and the step size in this numerical experiment were chosen deliberately in order to compare all flows at the same timescale.
It is important to note we numerically observe that a posterior with sufficiently narrow ridges causes sSVGD to produce overflow errors if the stepsize is not sufficiently small (even for two-dimensional cases).
On the contrary, sSVN inherits the desirable affine invariance of SVN, allowing for larger stepsizes that are more robust to narrow ridges.
For the posteriors studied in this paper, we observe that $\tau=0.1$ is a good step size for sSVN, and $\tau=0.01$ for sSVGD.
Our subsequent experiments use these values.

\subsection{Five-dimensional test case}

\begin
{figure}[tb]
\centering
   \includegraphics{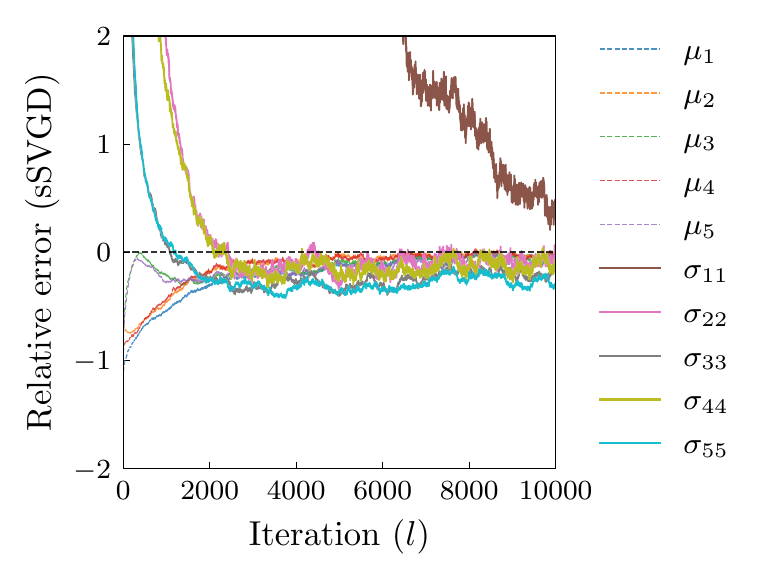} \\
   \includegraphics{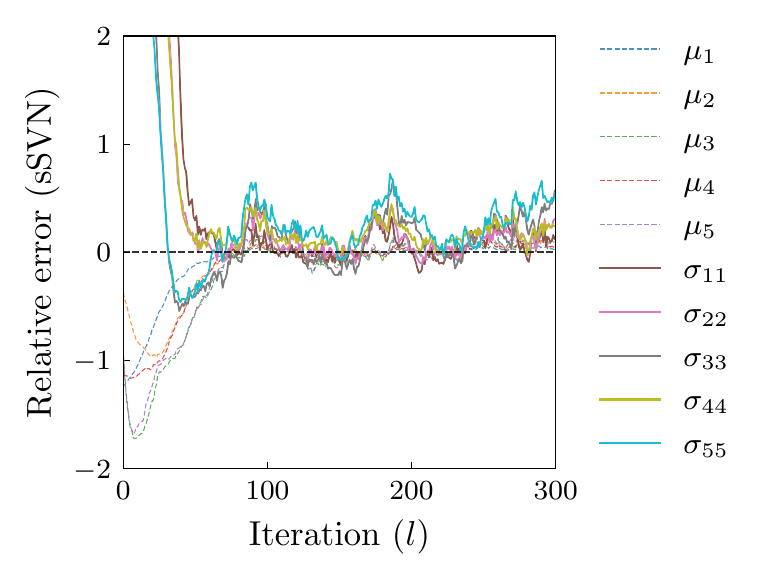}
\caption{Five-dimensional experiments. Evolution of moments for sSVGD (top) and sSVN (bottom).}
\label{fig:5d-moments}
\end{figure}

In this experiment we use a Hybrid Rosenbrock with parameters $n_2=2$, $n_1=3$, $a=10$, $b=30$ and compare only sSVGD and sSVN, using $N=100$ particles.
We display the flow of the mean and variances of the ensemble to their converged values in \cref{fig:5d-moments}.
Interestingly, sSVGD requires over $10^4$ iterations for $\sigma_{11}$ to converge, wheras all moments converge rapidly within $100$ iterations in sSVN.
We collect samples from the last $100$ iterations of both sSVGD and sSVN---neglecting issues related to sample autocorrelation for simplicity---and compare them to an i.i.d sample set of the same size in \cref{fig:5d-corner}.
The samples drawn from both sSVGD and sSVN appear to resemble the posterior well; however, sSVN exhibits an order of magnitude advantage in the number of iterations required.
Since $N=100$, this corresponds to an $\order(10^3)$ reduction in the number of gradient evaluations of the log-likelihood evaluations performed.

\begin{figure}[!tb]
\centering
\includegraphics[width=0.45\textwidth]{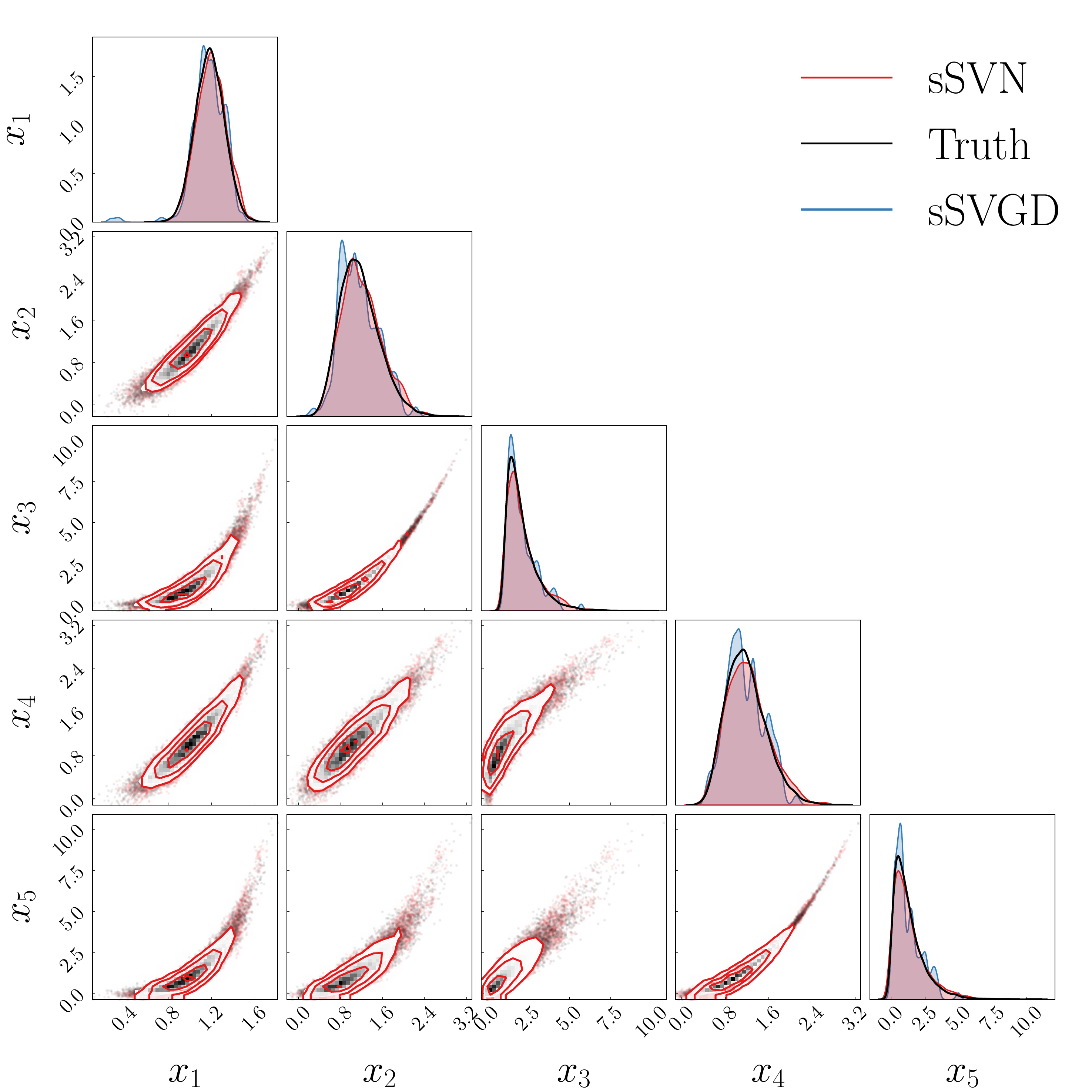}
\caption{
Five-dimensional Hybrid Rosenbrock corner plot comparing ground truth (black) with the output of sSVN (red) and sSVGD (blue). The diagonal figures compare one-dimensional KDEs of the marginals for sSVN, sSVGD, and ground truth. The off-diagonal figures compare only sSVN to ground truth for clarity, by plotting samples as well as contours in increments of $\sigma/2$ (sSVN) and a two-dimensional histogram of the density (ground truth). 
}
\label{fig:5d-corner}
\end{figure}

\subsection{Ten-dimensional test case}

\begin{figure}[tb!]
\centering
\begin{subfigure}[b]{0.4\textwidth}
   \includegraphics[width=1\linewidth]{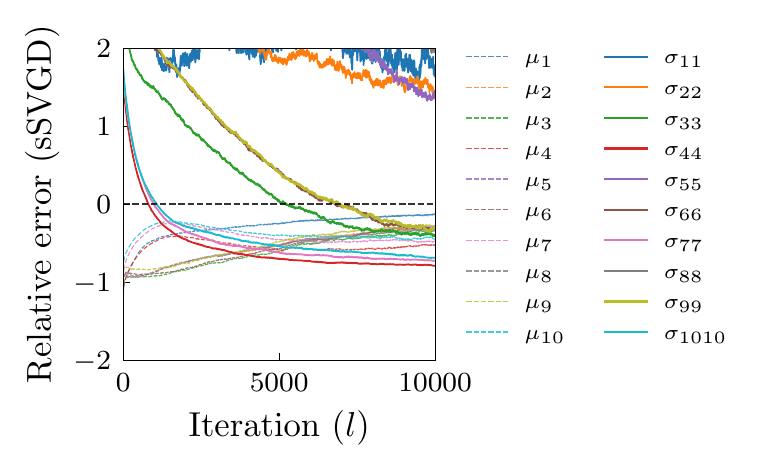}
   \label{subfig:10d-moments-ssvgd}
\end{subfigure}
\begin{subfigure}[b]{0.4\textwidth}
   \includegraphics[width=1\linewidth]{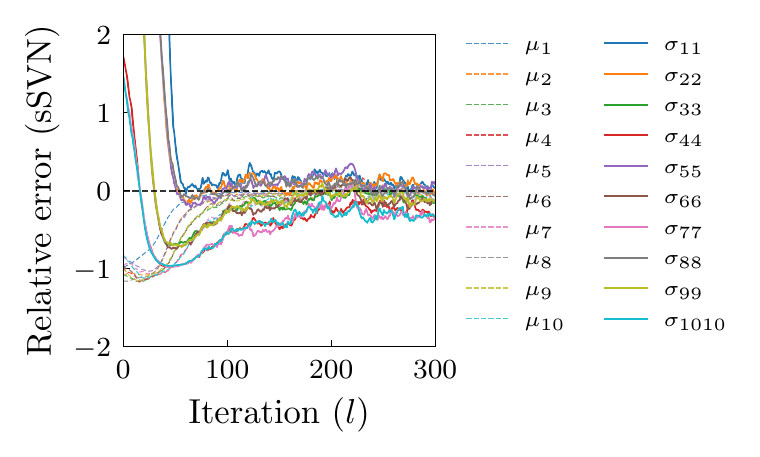}
   \label{subfig:10d-moments-ssvn}
\end{subfigure}
\begin{subfigure}[b]{0.4\textwidth}
   \includegraphics[width=1\linewidth]{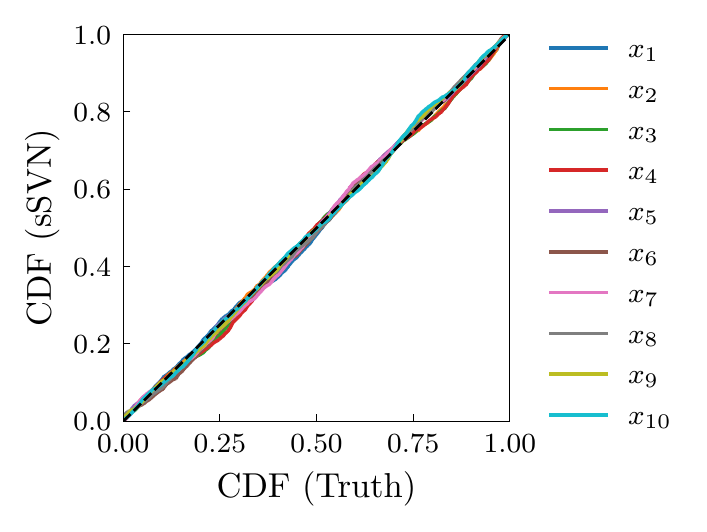}
   \label{subfig:10d-ppplot-ssvn}
\end{subfigure}
\caption{
Results of ten-dimensional Hybrid Rosenbrock run. (Top) Evolution of means and variances under sSVGD dynamics. (Middle) Evolution of means and variances under sSVN dynamics. (Bottom) P-P plot of samples collected from iterations $200-300$ compared to two-million ground truth samples.} \label{fig:10d-details}
\end{figure}

In this experiment we again use a Hybrid Rosenbrock, however with parameters $n_2=3$, $n_1=4$, $a=30$, $b=20$.
We begin with $N=300$ particles drawn from $\unif[-6,6]$, and run sSVGD and sSVN with $L=10000, 300$ iterations respectively with an identity metric kernel.
We plot the moment evolution in \cref{fig:10d-details}.
Similar to the five dimensional case, it appears that sSVN equilibrates after $200$ iterations, while sSVGD struggles.
As a measure of sample quality, we also present a P-P plot of the sSVN samples versus ground-truth, and find that they are in excellent agreement.
To further investigate the convergence of sSVN, \cref{fig:10d-corner} collects samples from the final $100$ iterations and presents a corner plot comparing to $30000$ ground truth samples.
We also compare the results of sSVN to sSVGD and ground truth in the one-dimensional marginals on the diagonal.
The sSVN samples reconstruct the posterior accurately, as opposed to sSVGD, which has not yet converged.

\begin{figure*}[!htb]
\centering
\includegraphics[width=1\linewidth]{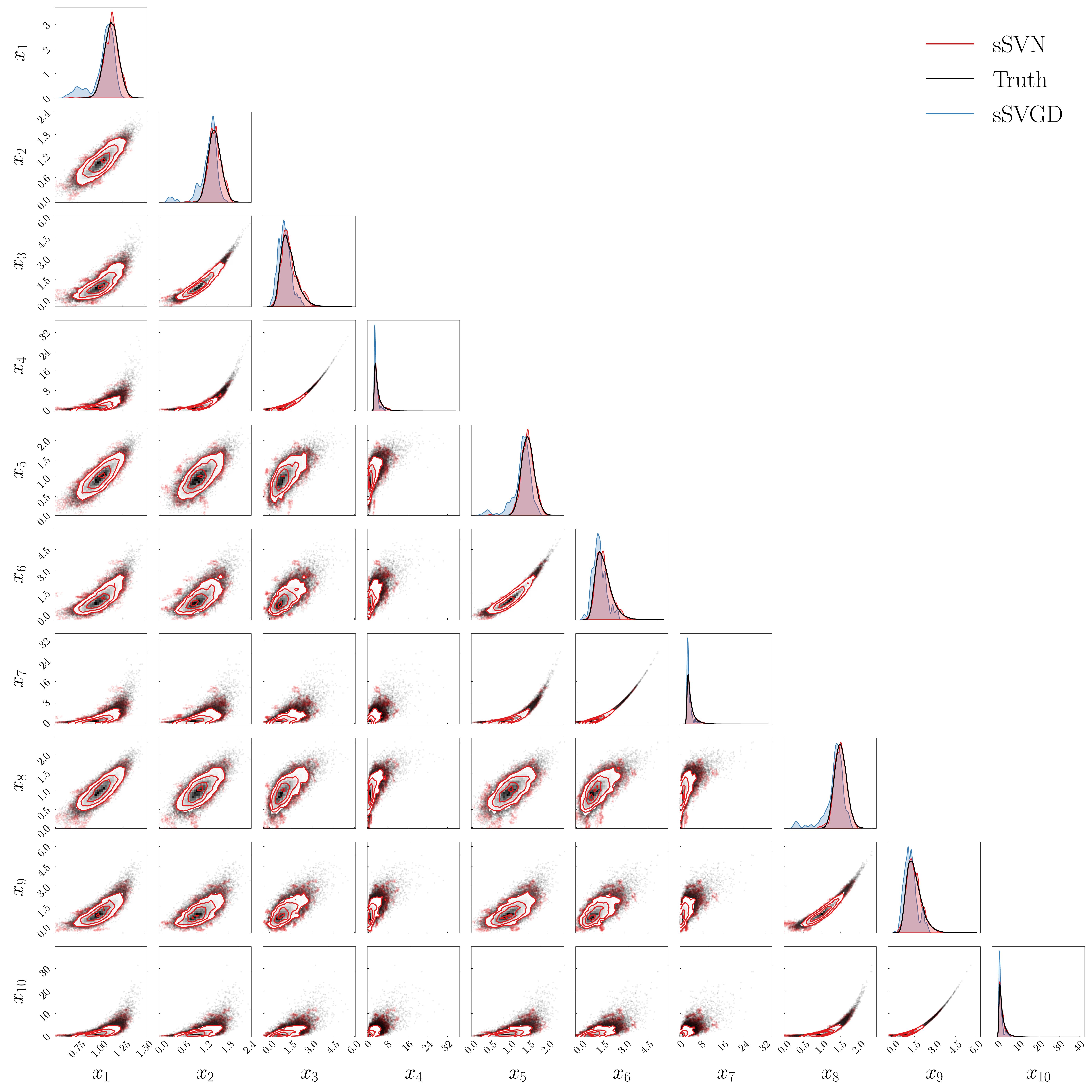}
\caption{Same as \cref{fig:5d-corner}, for the ten-dimensional Hybrid Rosenbrock density.}.
\label{fig:10d-corner}
\end{figure*}


\section{Discussion and further work} \label{sec:conclusion}
In contrast to its deterministic counterpart (SVN), sSVN may be made asymptotically correct by incorporating either a damping schedule or a Metropolis-Hastings correction, leading to a promising new approach to solving Bayesian inference tasks that require high-precision posterior reconstruction.
To demonstrate the performance of our proposed algorithm, we examined the flows and sample quality on a difficult class of test problems---the Hybrid Rosenbrock density---and showed that sSVN successfully reconstructs the posterior with at least three orders of magnitude fewer gradient evaluations of the log-likelihood than sSVGD.
In future work, it will be interesting to compare sSVN to other state of the art sampling algorithms such as dynamic nested sampling \cite{higsonDynamicNestedSampling2019,skillingNestedSamplingGeneral2006a} and Hamiltonian Monte Carlo \cite{betancourt2018conceptual, neal2011mcmc}.
Further improvements to the algorithm are possible as well.
For example, a Metropolis-Hastings correction may be implemented to eliminate all bias in the underlying Markov chain, random Fourier-feature kernels \cite{liu2018stein} may be used to improve the descent direction, and gradient evaluations of the log-likelihood may be used to mitigate issues related to auto-correlation \cite{riabiz2022optimal, hawkins2022online}.

\noindent {\bf Acknowledgements}
We thank Bassel Saleh and Peng Chen for discussions of Stein Variational methods at the outset of this work.
A.Z.~is supported by NSF Grant Number PHY-1912578.

\appendix
\bibliographystyle{amsplain}
\bibliography{main.bbl}
\onecolumn

\section{SVN simulates WNF}\label{sec:WNF}
Recall from \citet{wangInformationNewtonFlow2020} that the WNF direction is a conservative vector field $w$ satisfying the following equation
\begin{align*}
\nabla_i \nabla_j \brk[r]{\rho \nabla_i w_j} - \nabla_i \brk[r]{\rho w_j \nabla_i \nabla_j f} - \nabla_i \brk[r]{\rho \nabla_i f} - \lap \rho = 0 \:, \numberthis \label{eq:wnf-pde}
\end{align*}
and that the Wasserstein gradient flow direction is a conservative velocity field defined by $\vgf \defn -\nabla f - \nabla \ln \rho$, where the \textit{potential} $f$ is a scalar field.
Note that we are using Einstein summation convention, where repeated indices are summed over.
Therefore
\begin{align*}
\nabla_i \brk[r]{\rho \vgf_i}
= -\nabla_i \brk[s]!{\rho \brk[r]{\nabla_i f + \nabla_i \ln \rho}}
= -\nabla_i(\rho \nabla_i f) - \lap \rho \:,
\end{align*}
where $\lap$ denotes the Laplacian. Plugging this back into \cref{eq:wnf-pde} yields
\begin{align*}
0 &= \nabla_i \nabla_j (\rho \nabla_i w_j) - \nabla_i (\rho w_j \nabla_i \nabla_j f) + \nabla_i(\rho \vgf_i) \\
&= \nabla_i \brk[s]!{\rho \nabla_j \ln \rho \nabla_i w_j + \rho \nabla_i \nabla_j w_j - \rho w_j \nabla_i \nabla_j f + \rho \vgf_i} \\
&= \nabla_j \ln \rho \nabla_i w_j + \nabla_i \nabla_j w_j - w_j \nabla_i \nabla_j f + \vgf_i + \frac{\xi_i}{\rho} \\
&= -(\wnfhess w)_i + \vgf_i + \frac{\xi_i}{\rho} \:, \numberthis \label{eq:wnf-gf-mod}
\end{align*}
where $\xi$ is a divergence-free vector field, and we have defined
\begin{equation}\label{}
(\wnfhess w)_i = - \nabla_j \ln \rho \nabla_i w_j - \nabla_i \nabla_j w_j + w_j \nabla_i \nabla_i f.
\end{equation}
\cref{eq:wnf-gf-mod} illustrates how the Newton flow $w$ is related to the gradient flow $\vgf$.

We will now show that the weak form of \cref{eq:wnf-gf-mod} is releated to the variational characterization of the SVN direction given in Theorem 1 of \citet{detommasoSteinVariationalNewton2018a}.
Suppose $v$ is a conservative vector field restricted to the vector-valued reproducing kernel Hilbert space (RKHS) defined by $\rkhsd \defn \rkhs \times \rkhs \times \cdots \times \rkhs$, where $\rkhs$ is an RKHS with kernel $k$.
Let $T$ denote the embedding operator, which is defined by $(Tv)(y) = \int \rho(x)v(x)k(x,y)\dfl{x}$, and let $\vlrho$ denote the space of vector fields with finite squared norm with respect to $\rho$.
Then by the embedding property \cite{steinwart2008support} we have
\begin{equation}\label{eq:wnf-solve}
\brk[a]1{\wnfhess w - \vgf, v}_{\vlrho} = \brk[a]1{T (\wnfhess w) - T \vgf, v}_{\rkhsd} = \int \xi \cdot v \dfl{x} = 0 \:,
\end{equation}
where the last equality holds since $\xi$ and $v$ are orthogonal.
It has been shown in Theorem 2 of \citet{liuUnderstandingAcceleratingParticlebased2019} that $\brk[a]1{\vgf, v}_{\vlrho} = \brk[a]1{\vsvgd, v}_{\rkhsd}$.
Building on this, we have
\begin{align*}\label{eq:wnf-svn}
\brk[s]!{T(\wnfhess w)}_i(y) &= \int \rho(x)\brk[s]!{- \nabla_j \ln \rho(x) \nabla_i w_j(x) - \nabla_i \nabla_j w_j(x) + w_j(x) \nabla_i \nabla_j f(x)} k(x,y) \dfl{x} \\
&= \int \brk[s]!{- \nabla_j \rho(x) \nabla_i w_j(x) k(x,y) - \rho(x) \nabla_i \nabla_j w_j(x) k(x,y) + \rho(x) \brk[a]1{w_j(\cdot), k(x, \cdot)}_{\rkhs} \nabla_i \nabla_j f(x) k(x,y)} \dfl{x} \\
&=\int \rho(x) \nabla_i w_j(x) (\nabla_1)_j k(x,y) \dfl{x} + \brk[a]1{w_j(\cdot), \int \rho(x) \nabla_i \nabla_j f(x) k(x,\cdot) k(x,y) \dfl{x}}_{\rkhs} \\
&=\brk[a]1{w_j(\cdot), \expv_{x\sim \rho} \brk[s]!{(\nabla_1)_i k(x, \cdot) (\nabla_1)_j k(x,y) + \nabla_i \nabla_j f(x) k(x,\cdot) k(x,y)}}_{\rkhs} = \brk[a]!{w_j(\cdot), h_{ij}(\cdot, y)}_{\rkhs},
\numberthis
\end{align*}
where the last equality defines $h(\cdot, y)$, which is equivalent to the SVN Hessian of \cref{eq:blocks} when evaluated at two particles $z_m$ and $z_n$, and when identifying $f=-\ln\pi$.
The results \cref{eq:wnf-svn} and \cref{eq:wnf-solve}, together with this definition of $f$, reproduce the characterization of the SVN direction in \citet{detommasoSteinVariationalNewton2018a}\footnote{Where we have corrected a typo in Theorem 1 of \citet{detommasoSteinVariationalNewton2018a}.}.
Thus, we may understand SVN as an RKHS approximation to WNF which does not enforce that the velocity field be conservative.
\begin{remark}[]\label{}
This result allows us to understand sSVN as an MCMC algorithm that uses as its proposal an approximated Wasserstein Newton step, plus a random forcing term which depends on the Hessian.
In this way, sSVN is very similar to stochastic Newton (SN) discussed in \citet{martin2012stochastic} and \cref{sec:related-work}.
However, a key distinction is that sSVN yields an interacting particle system and performs its optimization on the space of probability measures, as opposed to parameter space.
\end{remark}

\section{Proofs}
\label{sec:Proofs}
\begin{figure}
\centering
\begin{subfigure}[]{.5\columnwidth}
  \centering
  \includegraphics{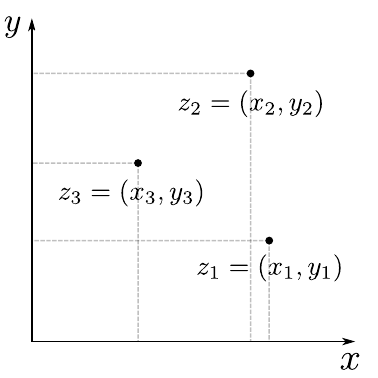}
\end{subfigure}%
\begin{subfigure}[]{.3\columnwidth}
  \centering
$$
\left[
\begin{array}{ccc}
x_1\\
y_1\\
\hline
x_2\\
y_2\\
\hline
x_3\\
y_3
\end{array}
\right]
$$
  \caption{}
  \label{fig:1-phi}
\end{subfigure}%
\begin{subfigure}[]{.3\columnwidth}
  \centering
$$
\left[
\begin{array}{ccc}
x_1\\
x_2\\
x_3\\
\hline
y_1\\
y_2\\
y_3
\end{array}
\right]
$$
  \caption{}
  \label{fig:1-gamma}
\end{subfigure}
\caption{The index functions $\phi(m,i)$ and $\gamma(i,m)$ return the index of the $i^{\text{th}}$ coordinate of particle $m$ in the particle and dimension ordering respectively. For example, if $N=3, d=2$, then the vector $z$ in the particle ordering (a) has $z_{\phi(2,1)}=x_2$, while in the dimension ordering (b) has $z_{\gamma{(1,2)}}=y_1$.}
\label{fig:coordinate-to-particle}
\end{figure}
Please refer to \cref{fig:coordinate-to-particle} for the definitions of the index maps $\phi$ and $\gamma$.
\subsection{Proof of \cref{prop:SVGD-MCMC-equivalence}}
\begin{proof}
We begin by taking \cref{eq:SVGDrecipeform} and expressing it in index notation
\begin{align*}
\diff{z_{\phi(m,i)}}{t} = \sum_{j,n} \brk[s]1{K_{\phi(m,i)\phi(n,j)} \nabla_{\phi(n,j)}  \ln \pi+ \nabla_{\phi(n,j)} K_{\phi(m,i)\phi(n,j)}} \,,
\numberthis
\end{align*}
where repeated indices are summed.
The application of $K$ on any vector, and in particular $\nabla \ln \pi$, is to kernel average each particle block. In order to see this we proceed in index notation
\begin{align*}
 \sum_{j,n} K_{\phi(m,i)\phi(n,j)} \nabla_{\phi(n,j)} \ln \pi
&= \frac{1}{N} \sum_{j,n} \brk[s]1{\kgram_{mn} \delta_{ij} \nabla_j \ln \pi(z_n)}
= \frac{1}{N}\sum_{n}\kgram_{mn} \nabla_i \ln \pi(z_n) \,.
\numberthis
\end{align*}
This shows the equivalence of the first terms on the right hand side of \cref{eq:SVGD-direction} and \cref{eq:SVGDrecipeform},
which is the driving force of SVGD.
Likewise, we have
\begin{align*}
 \sum_{j,n}\nabla_{\phi(n,j)}K_{\phi(m,i)\phi(n,j)}
&= \frac{1}{N} \sum_{j,n}
\brk[c]2{\delta_{ij} \brk[s]1{(\nabla_1)_j\, \kgram_{mn} \delta_{mn} + (\nabla_2)_j\, \kgram_{mn}}}
= \frac{1}{N}(\nabla_1)_i \kgram_{mm} + \frac{1}{N}\sum_n(\nabla_{2})_i \kgram_{mn} \,,
\numberthis
\end{align*}
which shows the equivalence of the second terms on the right hand side of of \cref{eq:SVGD-direction} and \cref{eq:SVGDrecipeform} when the kernel has a ``flat top''. In other words, for any $1 \le m \le N$, $\nabla_1 k(z_m, z_m) = 0$.
Finally, we show that $K$ inherits its positive (semi) definiteness from the kernel gram matrix $\kgram$.
Since $K_{\phi(m,i)\phi(n,j)} \defn \kgram_{mn} \delta_{ij}$ remains the same upon exchanging $m,n$ and $i,j$, $K$ is symmetric.
Finally, recall that $K$ is orthogonal to $D_K$, and therefore both matrices have identical eigenvalues. Furthermore, since $D_K$ is block diagonal, its eigenvalues are equivalent to the eigenvalues of $\kgram$, repeated $N$ times. Therefore $K$ inherits its definiteness from $\kgram$.
\end{proof}

\subsection{Proof that $P K P^{\top} = D_K$}
Clearly it is important to recognize that $P K P^{\top} = D_K$ in order for sSVGD to be tractable: instead of calculating the Cholesky decomposition of an $Nd \times Nd$ matrix, this result suggests that only a $N \times N$ decomposition is necessary. This may be seen as a consequence of the following result.
\begin{lemma}[Basis transformation]\label{}
Let $P$ be the permutation matrix which takes a vector $v \in \reals^{Nd}$ in the particle representation to the dimension representation. Then $M' = P M P^{\top}$, where
\begin{align*}
M &=
\left[
\begin{array}{ccc|c|ccc}
M_{11}^{11} & \cdots & M^{11}_{1D} & \hspace*{3mm} & M^{1N}_{11} & \cdots & M^{1N}_{1D} \\
\vdots & \ddots & \vdots&\hdots& \vdots & \ddots & \vdots \\
M^{11}_{D1} & \cdots & M^{11}_{DD} & & M^{1N}_{DD} & \cdots & M^{1N}_{DD} \\
\hline
&\vdots & & \ddots & &\vdots\\
\hline
\rule{0pt}{5mm} M^{N1}_{11} & \cdots & M^{N1}_{1D} & & M^{NN}_{11} & \cdots & M^{NN}_{1D} \\
\vdots & \ddots & \vdots&\hdots& \vdots & \ddots & \vdots \\
M^{N1}_{D1} & \cdots & M^{N1}_{DD} & & M^{NN}_{D1} & \cdots & M^{NN}_{DD}
\end{array}
\right]
&
M' &=
\left[
\begin{array}{ccc|c|ccc}
M_{11}^{11} & \cdots & M^{1N}_{11} & \hspace*{3mm} & M^{11}_{1D} & \cdots & M^{1N}_{1D} \\
\vdots & \ddots & \vdots&\hdots& \vdots & \ddots & \vdots \\
M^{N1}_{11} & \cdots & M^{NN}_{11} & & M^{N1}_{1D} & \cdots & M^{NN}_{1D} \\
\hline
&\vdots & & \ddots & &\vdots\\
\hline
\rule{0pt}{5mm} M^{11}_{D1} & \cdots & M^{1N}_{D1} & & M^{11}_{DD} & \cdots & M^{1N}_{DD} \\
\vdots & \ddots & \vdots&\hdots& \vdots & \ddots & \vdots \\
M^{N1}_{D1} & \cdots & M^{1N}_{D1} & & M^{N1}_{DD} & \cdots & M^{NN}_{DD}
\end{array}
\right]
\end{align*}
and $M^{mn}_{ij} \defn M_{\phi(m, i)\phi(n,j)}$.
\end{lemma}
\begin{proof}
Note that $P M P^{\top} = \brk[r]1{P(PM)^{\top}}^{\top}$. Beginning with $M$ we have
\begin{align*}
\xrightarrow{\textrm{Apply} P}
\left[
\begin{array}{ccc|c|ccc}
M_{11}^{11} & \cdots & M^{11}_{1D} & \hspace*{3mm} & M^{1N}_{11} & \cdots & M^{1N}_{1D} \\
\vdots & \ddots & \vdots&\hdots& \vdots & \ddots & \vdots \\
M^{N1}_{11} & \cdots & M^{N1}_{1D} & & M^{NN}_{11} & \cdots & M^{NN}_{1D} \\
\hline
&\vdots & & \ddots & &\vdots\\
\hline
\rule{0pt}{5mm} M^{11}_{D1} & \cdots & M^{11}_{DD} & & M^{1N}_{D1} & \cdots & M^{1N}_{DD} \\
\vdots & \ddots & \vdots&\hdots& \vdots & \ddots & \vdots \\
M^{N1}_{D1} & \cdots & M^{N1}_{DD} & & M^{NN}_{D1} & \cdots & M^{NN}_{DD}
\end{array}
\right]
\xrightarrow{\textrm{Transpose}}
\left[
\begin{array}{ccc|c|ccc}
M_{11}^{11} & \cdots & M^{N1}_{11} & \hspace*{3mm} & M^{11}_{D1} & \cdots & M^{N1}_{D1} \\
\vdots & \ddots & \vdots&\hdots& \vdots & \ddots & \vdots \\
M^{11}_{1D} & \cdots & M^{N1}_{1D} & & M^{11}_{DD} & \cdots & M^{N1}_{DD} \\
\hline
&\vdots & & \ddots & &\vdots\\
\hline
\rule{0pt}{5mm} M^{1N}_{11} & \cdots & M^{NN}_{11} & & M^{1N}_{D1} & \cdots & M^{NN}_{D1} \\
\vdots & \ddots & \vdots&\hdots& \vdots & \ddots & \vdots \\
M^{1N}_{1D} & \cdots & M^{NN}_{1D} & & M^{1N}_{DD} & \cdots & M^{NN}_{DD}
\end{array}
\right]
\end{align*}
\begin{align*}
\xrightarrow{\textrm{Apply} P}
\left[
\begin{array}{ccc|c|ccc}
M_{11}^{11} & \cdots & M^{N1}_{11} & \hspace*{3mm} & M^{11}_{D1} & \cdots & M^{N1}_{D1} \\
\vdots & \ddots & \vdots&\hdots& \vdots & \ddots & \vdots \\
M^{1N}_{11} & \cdots & M^{NN}_{11} & & M^{1N}_{D1} & \cdots & M^{NN}_{D1} \\
\hline
&\vdots & & \ddots & &\vdots\\
\hline
\rule{0pt}{5mm} M^{11}_{1D} & \cdots & M^{11}_{DD} & & M^{11}_{DD} & \cdots & M^{N1}_{DD} \\
\vdots & \ddots & \vdots&\hdots& \vdots & \ddots & \vdots \\
M^{1N}_{1D} & \cdots & M^{NN}_{1D} & & M^{1N}_{DD} & \cdots & M^{NN}_{DD}
\end{array}
\right]
\xrightarrow{\textrm{Transpose}}
\left[
\begin{array}{ccc|c|ccc}
M_{11}^{11} & \cdots & M^{1N}_{11} & \hspace*{3mm} & M^{11}_{1D} & \cdots & M^{1N}_{1D} \\
\vdots & \ddots & \vdots&\hdots& \vdots & \ddots & \vdots \\
M^{N1}_{11} & \cdots & M^{NN}_{11} & & M^{N1}_{1D} & \cdots & M^{NN}_{1D} \\
\hline
&\vdots & & \ddots & &\vdots\\
\hline
\rule{0pt}{5mm} M^{11}_{D1} & \cdots & M^{1N}_{D1} & & M^{11}_{DD} & \cdots & M^{1N}_{DD} \\
\vdots & \ddots & \vdots&\hdots& \vdots & \ddots & \vdots \\
M^{N1}_{D1} & \cdots & M^{1N}_{D1} & & M^{N1}_{DD} & \cdots & M^{NN}_{DD}
\end{array}
\right]
\end{align*}
\end{proof}
The result follows from observing that the $\delta_{ij}$ term in $K$ kills off all off diagonal blocks.

\subsection{Proof of \cref{thm:SVN-recipe}}
\begin{proof}
We show that \cref{eq:svnsde} may be derived from \cref{eq:ito-recipe} with diffusion matrix $\dsvn \defn N K H^{-1} K$, and curl matrix $Q=0$. Direct substitution yields
\begin{align*}
\dfl{z} &= \brk[r]{\dsvn \nabla \ln \pi + \nabla \cdot \dsvn}\dfl{t} + \sqrt{2 \dsvn} \dfl{B} \\
&= \brk[s]1{(NKH^{-1}K) \nabla \ln \pi + \nabla \cdot \brk[r]{N K H^{-1} K}}\dfl{t} + \sqrt{2 \dsvn} \dfl{B} \,.
\end{align*}
We focus only on the drift term and proceed in index notation, expanding the divergence and collecting terms
\begin{align*}
N K_{ab} H^{-1}_{bc} K_{ce} \nabla_e \ln \pi + N \nabla_e(K_{a b} H^{-1}_{bc} K_{ce})
&= \overbrace{N K_{ab}H^{-1}_{bc}(K_{ce}\nabla_e \ln \pi + \nabla_e K_{ce})}^{\vsvn}
+ \overbrace{\nabla_e(K_{ab} H^{-1}_{bc}) K_{ce}}^{\vdet} \,,
\numberthis \label{eq:svn-corrected-update}
\end{align*}
which gives the stated form of $\vdet$.
Finally, since $K$ is invertible, $KH^{-1}K$ is a congruence transformation of $H^{-1}$, and thus if $H^{-1}$ is positive definite, so is $\dsvn$.
\end{proof}
\begin{remark}
The assumption that $H$ be strictly positive definite is necessary to ensure that $\vsvn$ leads to a valid search direction.
Indeed, this is not true in general, and requires careful consideration.
See the discussion preceding \cref{thm:SVN-recipe} and Levenberg damping for a discussion on how to ensure $H$ is positive definite.
\end{remark}
\begin{remark}[]\label{}
We may express $\vdet$ in the form
\begin{align*}
\vdet_a & =N[(\nabla_e(K_{ab}) H^{-1}_{bc}K_{ce}
- K_{ab}H^{-1}_{bf} (\nabla_e H_{fg})  H^{-1}_{g c} K_{ce}]
\end{align*}
using the identity
$\nabla_e H^{-1}_{bc} = - H_{b f}^{-1} (\nabla_e H_{fg}) H^{-1}_{f c}$.
This expression illustrates that third derivatives are required to compute $\vdet$ in general, and second derivatives are required even with the Gauss Newton approximation. Note that similar terms appear in other higher-order flows, and are often neglected. See the paragraph on asymptotic correctness in \cref{subsec:practical-algorithm} for further discussion.
\end{remark}

\section{Scaling improvements} \label{sec:scaling}
In \cref{algo:ssvn} we form the SVN Hessian $H$, find the Cholesky decomposition of $H$, and use the decomposition to find the descent direction and calculate the noise.
However, for large $N, d$, the Cholesky decomposition and the associated storage requirements may be costly.
In this section we propose an alternative scheme utilizing conjugate gradients---whose iterations may be terminated when desired, and only requires the ability to evaluate matrix-vector products, thus circumventing the need to form $H$.

Observe that
\begin{align*}
z^{l+1} &= z^l + \tau N K H^{-1} \vsvgd + \sqrt{\tau}\gauss(0, 2 N K H^{-1} K) \\
&= z^l + \tau N K H^{-1} \vsvgd  + \sqrt{2 N \tau}KH^{-1} H \gauss(0,H^{-1}) \\
&= z^l + \tau NKH^{-1}\brk[s]!{\vsvgd + \sqrt{\frac{2}{N \tau}} \gauss(0, H)} \numberthis \label{eq:scalable-sSVN}
\end{align*}
We now show that $\gauss(0, H)$ may be sampled efficiently.
To begin, note that the SVN Hessian may be re-expressed as
\begin{align*}
H = H^{(1)} + H^{(2)} = N K H_{\pi}K + H^{(2)} \,,
\numberthis
\end{align*}
where $H_{\pi}, H^{(2)} \in\realsndnd$ are both block diagonal matrices whose $m^{\rm th}$ $d \times d$ blocks are given by $\nabla^2 \ln \pi(z_m)$, and $\frac{1}{N}\sum_n \nabla_1 k(z_m, z_n) \nabla_1^{\top} k(z_m, z_n)$ respectively. Indeed, we can see by
\begin{align*}
N K H_{\pi}K &= N K_{\phi(m,i)\phi(l,e)} (H_{\pi})_{\phi(l,e)\phi(p,f)} K_{\phi(p,f)\phi(n,j)} \\
&= \frac{1}{N} \sum_{e,f} \sum_{l,p} \brk[s]1{\delta_{lp} \kgram_{ml}  \kgram_{pn} \delta_{ie} \delta_{fj} \nabla_{e} \nabla_{f} \ln \pi (z_l) }\\
&=\frac{1}{N} \sum_{l} \brk[s]1{\kgram_{ml} \kgram_{ln} \nabla_{i}\nabla_j \ln \pi(z_l)} \,,
\end{align*}
that $H^{(1)}$ corresponds to the first term in \cref{eq:blocks}. The noise term may thus be decomposed as
\begin{align*}
\gauss(0, H) = \gauss(0, N K H_{\pi} K + H^{(2)}) = \sqrt{N} K \gauss(0, H_{\pi}) + \gauss(0, H^{(2)}) \,. \numberthis \label{eq:noise-1}
\end{align*}
Now, since $H_{\pi}$ is block diagonal, the task of drawing the first term in \cref{eq:noise-1} decomposes into $N$ individual $d \times d$ subproblems, which may be solved with $N \times (d \times d)$ Cholesky decompositions. The second term simplifies as well. The $m^{\rm th}$ block of $\gauss(0, H^{(2)})$ may be evaluated with
\begin{align*}
\gauss \brk[r]2{0, \frac{1}{N}\sum_n \nabla_1 k(z_m, z_n) \nabla_1^{\top} k(z_m, z_n)} &= \frac{1}{\sqrt{N}}\sum_n \gauss \brk[r]2{0,  \nabla_1 k(z_m, z_n) \nabla_1^{\top} k(z_m, z_n)} \\
&= \frac{1}{\sqrt{N}} \sum_n \nabla_1 k(z_m, z_n) \gauss_{mn} \,, \numberthis \label{eq:noise-2}
\end{align*}
where for every $1\le m,n \le N$, $\gauss_{mn}$ denotes a sample from a standard normal.
Thus, a draw from $\gauss(0,H)$ has a cost of $\order(N d^3)$, as opposed to $\gauss(0,H^{-1})$, which has a (time complexity) cost of $\order(N^3 d^3)$.
Meanwhile, with the use of a CG solver, we no longer need to form the SVN Hessian, thus improving memory complexity by a factor of $N$.

This scheme is thus a simple modification of the original ``full SVN" algorithm. Namely, the only change is we instead work with a noise perturbed SVGD direction $\vsvgd_* = \vsvgd + \sqrt{2/(N \tau)} \gauss(0, H)$. A summary is provided in \cref{algo:ssvn-scale}.

\begin{algorithm2e}
\SetAlgoLined
\KwIn{Initialize ensemble $z^1$, $\tau>0$}
 \For{$l = 1, 2, \ldots, L$}{
Calculate perturbed SVGD direction $\vsvgd_*$ \cref{eq:scalable-sSVN} \;
Define method yielding Hessian-vector product $Hv$, $v \in \realsnd$ \;
CG solve: $H \vsvn_* = \vsvgd_*$ \;
$z^{l+1} \leftarrow z^l + \tau \vsvn_*(z^l)$ \;
 }
 \caption{Stochastic SVN (CG)}\label{algo:ssvn-scale}
\end{algorithm2e}

\paragraph{Remark on block diagonal SVN}
The block diagonal approximation, originally introduced in \citet{detommasoSteinVariationalNewton2018a}, considers only the diagonal blocks $h^{mm}$ where $1 \le m \le N$ of the Hessian $H$.
In turn, this decouples the solves necessary to perform \cref{eq:svn-linear-system} from an $Nd \times Nd$ to $N \times (d \times d)$ solves.
Finally, the approximation takes $\alpha$ to be the step direction directly, as opposed to using \cref{eq:svn-direction}.
This significantly reduces the burden of solving the linear system and improves scalability.
Unfortunately, this scheme does not naturally extend to the stochastic case.
Specifically, notice that if one were to design a stochastic block diagonal algorithm, the leading term would be $H^{-1} K \nabla \ln \pi$.
The primary issue with $H^{-1} K$ is that it is dense, and not symmetric.
The cost of the Cholesky decomposition remains $\order(N^3 d^3)$, at which point the full Hessian may as well be used.

\end{document}